\documentclass{article}

    \PassOptionsToPackage{numbers, compress}{natbib}
 \usepackage[preprint]{neurips_2026}

\usepackage[utf8]{inputenc} 
\usepackage[T1]{fontenc}    
\usepackage{hyperref}       
\usepackage{url}            
\usepackage{booktabs}       
\usepackage{amsfonts}       
\usepackage{nicefrac}       
\usepackage{microtype}      
\usepackage{xcolor}         

\usepackage{xcolor}
\usepackage{colortbl} 
\usepackage{url}

\usepackage{algorithmic}
\usepackage[ruled,vlined]{algorithm2e}
\usepackage[shortlabels]{enumitem}
\setlist[enumerate]{nosep}

\usepackage{amsmath}
\usepackage{mathtools}
\usepackage{amsthm}
\usepackage{bm}
\usepackage{mathrsfs}
\usepackage[capitalize,noabbrev]{cleveref}
\usepackage{multirow}
\usepackage{booktabs} 

\theoremstyle{plain}
\newtheorem{theorem}{Theorem}[section]

\newtheorem{lemma}[theorem]{Lemma}

\theoremstyle{definition}
\newtheorem{definition}{Definition}[section]

\theoremstyle{remark}

\usepackage{changepage}
\definecolor{myhighlight}{RGB}{220,240,255}
\usepackage{multicol}
\SetKwInput{KwInput}{Input}
\SetKwInput{KwOutput}{Output}
\usepackage{tcolorbox}
\usepackage{caption}
\usepackage{thm-restate}
\usepackage{wrapfig}
\usepackage{float}

\usepackage{amssymb}

\DeclareMathOperator{\clip}{clip}

\DeclareMathOperator{\mean}{mean}
\DeclareMathOperator{\std}{std}

\hypersetup{
	colorlinks=true
}

\tcbset{
  takeaway/.style={
    colback=gray!10!white,
    colframe=gray!60!black,
    fonttitle=\bfseries,
    coltitle=white,
    boxrule=1pt,
    arc=4pt,
    left=6pt,
    right=6pt,
    top=6pt,
    bottom=6pt,
    title=Takeaway,
  }
}
\tcbset{
  example/.style={
    colback=blue!10!white,  
    colframe=blue!60!black, 
    fonttitle=\bfseries,
    coltitle=white,
    boxrule=1pt,
    arc=4pt,
    left=6pt,
    right=6pt,
    top=6pt,
    bottom=6pt,
    title=Example,
  }
}

\usepackage{bm}
\usepackage{xspace}
\def\algo{POPO\xspace}
\usepackage{multirow}
\usepackage{colortbl}
\definecolor{myhighlight}{RGB}{220,240,255}
\usepackage{tcolorbox}
\usepackage{subfigure}

\title{RLVR without Ineffective Samples: Group Prioritized Off-Policy Optimization for LLM Reasoning}

\author{
Yixiu Mao,\quad Yun Qu,\quad Qi Wang\thanks{Corresponding authors.},\quad Heming Zou,\quad Xiangyang Ji\footnotemark[1]\\
Department of Automation, Tsinghua University\\
\texttt{\{myx21, qy22, zouhm24\}@mails.tsinghua.edu.cn} \\
\texttt{cheemswang@mail.tsinghua.edu.cn}, \quad \texttt{xyji@tsinghua.edu.cn} \\
}

\begin{document}

\maketitle

\begin{abstract}
Reinforcement learning with verifiable rewards (RLVR) has emerged as a powerful paradigm for enhancing the reasoning capabilities of large language models (LLMs). However, its effectiveness is substantially hindered by the prevalence of ineffective training data: many sampled prompts yield response groups that are either entirely correct or entirely incorrect, resulting in zero-variance rewards and limited learning signals. Recent state-of-the-art methods address this issue through extensive LLM rollouts to filter ineffective samples, but at the cost of considerable computational overhead. Alternative approaches, including predictive sampling and trajectory replay, aim to improve data efficiency but often remain insufficient and may introduce additional issues such as systematic bias or suboptimal constraints. To address these limitations, we propose Group Prioritized Off-Policy Optimization (POPO), a simple yet effective framework that fully exploits effective training batches without additional rollout overhead. POPO comprises two key components: prioritized group replay and decoupled off-policy optimization. The former replaces ineffective on-policy groups with effective off-policy groups via a recency-based replay mechanism that jointly considers sample quality and the degree of off-policiness. To further mitigate the off-policy gap, POPO employs decoupled importance sampling to correct off-policy bias while maintaining stable policy updates under consistent trust-region constraints. Empirical evaluations across diverse reasoning tasks, including mathematics, planning, and visual geometry, demonstrate that POPO substantially accelerates RL finetuning and achieves strong reasoning performance with significantly fewer rollouts.
\end{abstract}

\section{Introduction}

\begin{figure}[t]
    \centering
    \includegraphics[width=\linewidth]{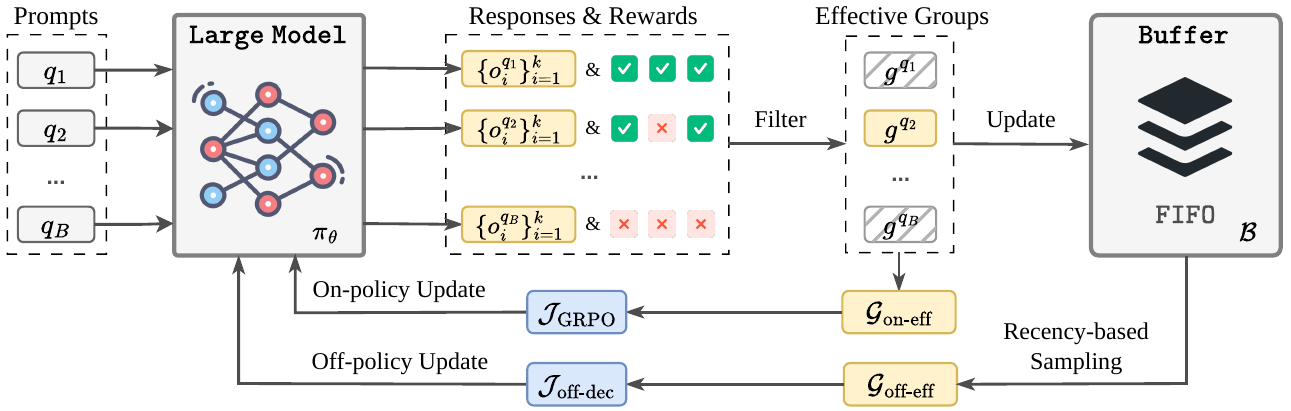}
    \caption{
    Framework overview of \algo.
    During RLVR training, ineffective online groups are replaced with high-quality off-policy groups via recency-based replay, forming batches composed entirely of effective samples without additional overhead. For replayed off-policy groups, decoupled importance sampling is introduced to correct off-policy bias and maintain a consistent trust-region constraint across on- and off-policy updates, enabling stable and efficient policy optimization.
    }
    \vspace{-7pt}
    \label{fig:framework}
\end{figure}

Reinforcement learning with verifiable rewards (RLVR) has emerged as a core technique for enhancing the reasoning capabilities of large language models (LLMs)~\citep{lightman2023let,jaech2024openai,guo2025deepseek,team2025kimi}. By encouraging LLMs to generate extended chain-of-thought (CoT) reasoning~\citep{wei2022chain} and optimizing policies using verifiable reward signals~\citep{jaech2024openai}, RLVR has enabled substantial progress in large reasoning models (LRMs) across a broad range of knowledge-intensive tasks, including scientific question answering~\citep{he2024olympiadbench}, symbolic mathematics~\citep{luo2025deepscaler}, logical deduction~\citep{xie2025logic}, and program synthesis~\citep{luo2025deepcoder}.

However, the effectiveness of RLVR is significantly affected by the prevalence of ineffective training samples~\citep{yu2025dapo}. In practice, a substantial portion of sampled prompts produce response groups that are either entirely correct or entirely incorrect~\citep{zhang2025srpo,qu2025can}.
This leads to zero-variance rewards and limits the availability of informative training signals~\citep{bae2025online, chen2025self}. 
To alleviate this issue, recent state-of-the-art (SOTA) methods perform extensive LLM rollouts over multiple candidate batches to identify and retain effective samples~\citep{yu2025dapo, cui2025process, liu2025prorl}. Although effective at improving sample quality, this strategy incurs substantial computational overhead that often exceeds the cost of RL finetuning itself.

Considering this efficacy-efficiency dilemma, several alternative approaches have been explored. Predictive sampling methods estimate prompt success probabilities prior to rollout and prioritize prompts with intermediate success rates~\citep{qu2025can,shen2025bots,xu2025single}.
However, accurate prediction remains challenging, especially for rapidly evolving models or large-scale training datasets. Consequently, these approaches often achieve only limited improvements in effective sample ratios.
Another line of research leverages trajectory replay to strengthen training signals from difficult prompts, typically inserting previously successful trajectories into online response groups under specific conditions~\citep{lu2025arpo,zhang2025rlep,zhan2025exgrpo}. However, trajectory replay does not fundamentally resolve prompt inefficiency, leaving a substantial proportion of training samples ineffective. Moreover, the resulting policy optimization objectives may suffer from issues including intractable bias, suboptimal constraints, or invalid advantage estimates.

To overcome these limitations, we propose Group \underline{P}rioritized \underline{O}ff-\underline{P}olicy \underline{O}ptimization~(POPO), a simple yet effective approach that stably optimizes entire batches of effective data without incurring additional rollout overhead. POPO consists of two components: prioritized group replay and decoupled off-policy optimization. 
The prioritized group replay mechanism jointly accounts for sample quality and off-policiness. Specifically, it replaces ineffective on-policy groups, each comprising a prompt and its associated responses, with effective off-policy groups generated by recent policies. This design not only eliminates ineffective training data but also avoids issues inherent in trajectory replay, such as intractable bias and invalid advantage estimates. Furthermore, we present a decoupled importance-sampling strategy for stable off-policy optimization over replayed groups. While prior off-policy objectives either introduce bias by treating replayed samples as on-policy or impose suboptimal constraints toward the behavior policy, \algo corrects off-policy bias and ensures stable policy updates through consistent trust-region constraints across both on-policy and off-policy samples.

Empirically, we evaluate \algo on a diverse set of downstream reasoning tasks, including competition-level mathematics, numerical planning, and visual geometry. The results show that \algo substantially accelerates RL finetuning, achieving performance comparable to the SOTA method while requiring substantially fewer rollouts.
Moreover, \algo consistently improves reasoning performance across varying response group sizes and base RLVR algorithms, and remains complementary to existing sampling strategies, highlighting its broad applicability and practical utility for RLVR training.

\section{Preliminary}
\label{sec:preliminary}
\textbf{RL with Verifiable Rewards~(RLVR).}
Let $q \sim \mathcal{D}$ denote a prompt sampled from a dataset and $o \sim \pi_{\theta}(\cdot|q)$ a response generated by policy $\pi_{\theta}$. RLVR aims to maximize the expected reward:
\begin{equation}
\max_{\theta\in\Theta} \; \mathbb{E}_{q \sim \mathcal{D}, o \sim \pi_{\theta}(\cdot|q)} \left[ r(q, o) \right],
\end{equation}
where the reward function $r(q, o)$ typically verifies response correctness, assigning a value of $1$ to correct outputs and $0$ otherwise. Such rewards are simple, reliable, and less susceptible to reward hacking in domains such as mathematical reasoning and code generation.

\textbf{Group Relative Policy Optimization~(GRPO).}
To optimize the above objective, GRPO~\citep{shao2024deepseekmath} is a widely adopted policy gradient method that avoids value function modeling. At each training step, a batch of $B$ prompts is sampled, and for each prompt $q$, the policy generates $k$ responses $\{o_i\}_{i=1}^k$, forming a group $g^q = \bigl(q, \{o_i\}_{i=1}^k \bigr)$. GRPO then maximizes the following objective:
\begin{equation}
\mathcal{J}_{\text{GRPO}}(\theta)  = \mathbb{E}_{q, o_i \sim \pi_{\theta_{\text{old}}}(\cdot | q)} \Bigg[
\frac{1}{k}\sum_{i=1}^k \frac{1}{|o_i|}\sum_{t=1}^{|o_i|} \min \Big( \rho_{i, t}(\theta,\theta_{\text{old}}) \hat{A}_i, \clip(\rho_{i, t}(\theta,\theta_{\text{old}}), 1 \pm \epsilon) \hat{A}_i \Big)\Bigg],
\label{eq:GRPO}
\end{equation}
where $\rho_{i, t}(\theta,\theta_{\text{old}}) = \frac{\pi_{\theta}\left(o_{i, t} | q, o_{i,<t}\right)}{\pi_{\theta_{\text{old}}}\left(o_{i, t} | q, o_{i,<t}\right)}$.
The clipping term enforces proximity between $\pi_{\theta}$ and $\pi_{\theta_{\text{old}}}$ for stable updates. The advantages are computed via group-wise normalization:
\begin{equation}
\hat{A}_{i} = \frac{r(q,o_i)-\mean\big(\{r(q,o_j)\}_{j=1}^k\big)}{\std\big(\{r(q,o_j)\}_{j=1}^k\big)}.
\end{equation}
This formulation eliminates value modeling and has shown strong performance on reasoning tasks.

\section{Ineffective Samples and Prior Solutions}
\label{sec:ineffective data}
This section characterizes ineffective RLVR samples and reviews prior solutions and their limitations.

\textbf{Ineffective Samples.}
In RLVR, not all samples contribute equally to learning. When a prompt consistently yields either all-correct or all-incorrect responses, the resulting rewards provide little informative signal for model optimization~\citep{chen2025self, yu2025dapo, bae2025online}. For algorithms such as GRPO, this leads to vanishing normalized advantages and consequently stalled optimization.
\begin{definition}[Ineffective Sample]
\label{def:effective}
Given a prompt $q$ and its response group $g^q = \bigl(q, \{o_i\}_{i=1}^k \bigr)$, we define the sample as ineffective if the reward set exhibits zero variance: $\std\bigl(\{r(q,o_i)\}_{i=1}^k\bigr) = 0$.
\end{definition}
Both prior studies~\citep{zhang2025srpo,yu2025dapo,qu2025can} and our empirical findings (Fig.~\ref{fig:effective ratio}) indicate that, under standard RLVR training, a considerable portion of samples tend to be ineffective and contribute little to learning. This observation has motivated a growing body of research aimed at mitigating such inefficiencies.

\textbf{Active Sampling.}
A common approach is to prioritize informative prompts via active sampling~\citep{zhang2025srpo, chen2025self, yu2025dapo, bae2025online}. A representative SOTA method is DAPO \citep{yu2025dapo}, which performs rollouts over an expanded candidate prompt batch $\hat{\mathcal{Q}}_t$ and filters out prompts whose responses yield identical rewards:
\begin{equation}
\mathcal{Q}_t = \left\{ q \in \hat{\mathcal{Q}}_t \;\middle|\;
\std\bigl(\{r(q,o_i)\}_{i=1}^k\bigr) > 0 \right\}.
\label{eq:ds}
\end{equation}
While effective at improving sample quality, this strategy incurs substantial computational overhead due to repeated LLM rollouts over a larger candidate pool.
The additional rollout cost can dominate overall training, particularly for reasoning tasks involving long CoT outputs.

To reduce rollout overhead, predictive sampling methods estimate prompt success probabilities and prioritize prompts with intermediate success rates before executing rollouts~\citep{qu2025can,shen2025bots,xu2025single}. However, accurate prediction remains challenging, particularly when (i) prompt success rates shift rapidly as the policy evolves, or (ii) historical information is sparse in low-epoch or large-scale dataset settings. As a result, these methods often yield only limited improvements in effective sample ratios. Moreover, prioritizing intermediate-difficulty prompts may systematically exclude highly challenging samples, which are important for improving the upper bound of reasoning capability~\citep{yan2025learning}.

\textbf{Trajectory Replay.}
Another line of work improves data efficiency through replay mechanisms. These methods typically store historical successful responses in a prompt-specific buffer and insert them into online response groups under specific conditions~\citep{lu2025arpo,zhang2025rlep,zhan2025exgrpo,zhang2025improving}, such as when all online rollouts fail~\citep{lu2025arpo,zhang2025improving}. Although effective at amplifying sparse success signals in difficult tasks, trajectory replay does not fully resolve the problem of ineffective samples. In particular, samples remain ineffective when (i) all online responses are correct, or (ii) no successful trajectory has ever been obtained for extremely hard prompts or in large-scale dataset settings.

Moreover, trajectory replay introduces additional challenges for policy optimization. Many methods directly apply on-policy objectives to replayed samples~\citep{lu2025arpo,zhang2025rlep}, leading to biased optimization when the behavior policy differs from the current policy. 
Recent work~\citep{zhan2025exgrpo,zhang2025improving} has attempted to correct this bias via importance sampling with respect to the behavior policy that generated the replayed sample. However, this remains imperfect as selectively replayed trajectories are drawn from a filtered distribution rather than the true behavior policy, introducing potentially intractable bias. Furthermore, constraining updates toward the behavior policy can lead to inconsistent trust-region constraints and may overly restrict policy improvement. Finally, mixing off-policy trajectories within groups invalidates normalized advantage estimates, as they no longer correspond to any well-defined policy.

\section{Group Prioritized Off-Policy Optimization}
To address the limitations of prior methods, we propose Group Prioritized Off-Policy Optimization (POPO), a simple yet effective framework that constructs fully effective training batches via replay, enabling stable off-policy optimization without additional rollout cost. POPO consists of two components: prioritized group replay (\cref{sec:prioritized group replay}) and decoupled off-policy optimization (\cref{sec:off-policy optimization}).

\subsection{Prioritized Group Replay}
\label{sec:prioritized group replay}

Existing trajectory replay methods face the aforementioned challenges primarily because they retain all sampled prompts and operate at the response level. Since extremely easy or hard prompts provide limited learning signals~\citep{chen2025self,qu2025can}, manipulating responses alone is often insufficient and may further introduce intra-group policy inconsistency. We therefore propose prioritized group replay, a group-level mechanism that replaces ineffective prompt-response groups with historical groups according to both quality and off-policiness criteria.

\textbf{Quality of Groups.}
Partially correct groups are generally of higher quality as they provide stronger learning signals~\citep{bae2025online,chen2025self}. 
At each training step, we sample a batch of $B$ prompts $\mathcal{Q}_{\text{on}} \subset \mathcal{D}$ and perform online rollouts, yielding online groups $\mathcal{G}_{\text{on}} = \left\{ g^q \;\middle|\; q \in \mathcal{Q}_{\text{on}} \right\}$. From these online groups, we retain only effective groups with non-zero reward variance:
\begin{equation}
\mathcal{G}_{\text{on-eff}} = \left\{ g^q \in \mathcal{G}_{\text{on}} \;\middle|\; \mathrm{std}\bigl(\{r(q,o_i)\}_{i=1}^k\bigr) > 0 \right\}.
\end{equation}
Similar to DAPO, this filtering step reduces the training batch size. Instead of continually rolling out additional prompts, we replenish the batch using a replay buffer $\mathcal{B}$ that stores previously observed effective groups.
As $|\mathcal{G}_{\text{on-eff}}| \leq B$, we sample $B - |\mathcal{G}_{\text{on-eff}}|$ effective groups from $\mathcal{B}$ according to a specified sampling strategy to form the off-policy effective group set $\mathcal{G}_{\text{off-eff}}$:
\begin{equation}
\mathcal{G}_{\text{off-eff}} = 
\mathrm{Sample}\!\left(\mathcal{B},\, B - |\mathcal{G}_{\text{on-eff}}|\right).
\end{equation}
The training batch is constructed as $\mathcal{G}_{\text{train}} =
\mathcal{G}_{\text{on-eff}} \,\cup\, \mathcal{G}_{\text{off-eff}}$.
This design ensures that every training batch consists entirely of effective groups, without requiring rollout-intensive rejection schemes such as DAPO.
Compared with predictive sampling methods, it also preserves exploration by sampling prompts uniformly rather than concentrating on specific types, such as medium-difficulty prompts.

\textbf{Off-policiness of Groups.}
While the quality criterion determines which groups enter the buffer, the off-policiness criterion determines which groups are replayed.
To mitigate distribution shift, replayed groups should be generated by policies close to the current policy. A direct approach would be to compute KL divergence for all candidate groups and then perform filtering, which is computationally expensive for LLMs.
Fortunately, under proximal policy updates, policies trained in more recent steps are expected to remain closer to the current policy. Lemma~\ref{lem:off-policiness} formalizes this intuition by showing that the total variation bound between policies grows linearly with the update-step difference.
\begin{lemma}
Under proximal policy updates, assume that $1-\epsilon \leq \frac{\pi_{i+1}(a|s)}{\pi_{i}(a|s)} \leq 1+\epsilon$ for all steps $i$ and state-action pairs $(s,a)$. Then, for any step difference $n\ge 1$, $\mathrm{TV}\big(\pi_{i}(\cdot|s),\pi_{i+n}(\cdot|s)\big) \le \frac{\epsilon n}{2}$.
\label{lem:off-policiness}
\end{lemma}
Motivated by this observation, we adopt a recency-based replay strategy. Specifically, we maintain the replay buffer $\mathcal{B}$ as a first-in-first-out (FIFO) queue and replay the most recently stored groups. Let $\mathcal{B} = (g_1, g_2, \ldots, g_{|\mathcal{B}|})$ denote the buffer ordered by insertion time. To control off-policiness, we deterministically select groups from the tail of the buffer:
\begin{equation}
\mathcal{G}_{\text{off-eff}} =
\left\{
g_{|\mathcal{B}|-i+1} \in \mathcal{B}
\;\middle|\;
i = 1, \ldots, B - |\mathcal{G}_{\text{on-eff}}|
\right\}.
\end{equation}
This strategy provides a lightweight approximation to explicit KL-based filtering without additional LLM likelihood computation. Moreover, since at most $B$ groups are replayed per step, the buffer capacity can be set to $B$, incurring minimal memory overhead.
After constructing the training batch, the replay buffer is updated by including all effective online groups at the current step: $\mathcal{B} \leftarrow \mathcal{B} \cup \mathcal{G}_{\text{on-eff}}$.

\subsection{Decoupled Off-Policy Optimization}
\label{sec:off-policy optimization}

Although prioritized replay limits the off-policy gap, a residual discrepancy remains. We introduce decoupled off-policy optimization to correct this bias while preserving stable trust-region updates.

Let $\pi_{\beta}$ denote the behavior policy that generated a replayed sample, i.e., the LLM policy at the time the sample was collected.
A straightforward approach, adopted by several existing replay-based 
methods~\citep{lu2025arpo,zhang2025rlep}, is to treat off-policy data as if they were on-policy. This yields an objective that shares the same form as 
$\mathcal{J}_{\text{GRPO}}$, but with the expectation taken under $\pi_{\beta}$ rather than $\pi_{\theta_{\text{old}}}$:
\begin{equation}
\mathcal{J}_{\text{off-}\pi_{\text{old}}}(\theta)  = \mathbb{E}_{q, o_i \sim \pi_{\beta}(\cdot | q)} \Bigg[
\frac{1}{k}\sum_{i=1}^k \frac{1}{|o_i|}\sum_{t=1}^{|o_i|} \min \Big( \rho_{i, t}(\theta,\theta_{\text{old}}) \hat{A}_i, \clip(\rho_{i, t}(\theta,\theta_{\text{old}}), 1 \pm \epsilon) \hat{A}_i \Big) \Bigg].
\label{eq:off_piold}
\end{equation}
However, the objective becomes biased whenever $\pi_{\beta} \neq \pi_{\theta_{\text{old}}}$. In typical replay scenarios, the behavior policy $\pi_{\beta}$ may deviate substantially from the current policy $\pi_{\theta_{\text{old}}}$, which introduces systematic bias into the optimization and can degrade policy updates.

To correct this off-policy mismatch, a second line of work~\citep{zhan2025exgrpo,sun2025improving,liang2025squeeze} replaces the denominator in the importance ratio with the behavior policy $\pi_{\beta}$, yielding the following objective:
\begin{equation}
\mathcal{J}_{\text{off-}\pi_\beta}(\theta) = \mathbb{E}_{q, o_i \sim \pi_{\beta}(\cdot | q)} \Bigg[
\frac{1}{k}\sum_{i=1}^k \frac{1}{|o_i|}\sum_{t=1}^{|o_i|} \min  \Big( \rho_{i, t}(\theta,\beta) \hat{A}_i, \clip(\rho_{i, t}(\theta,\beta), 1 \pm \epsilon) \hat{A}_i \Big) \Bigg],
\end{equation}
where we generalize the notation $\rho$ and define, for any $\pi_{\theta_1}$ and $\pi_{\theta_2}$, $\rho_{i, t}(\theta_1,\theta_2) = \frac{\pi_{\theta_1}\left(o_{i, t} | q, o_{i,<t}\right)}{\pi_{\theta_2}\left(o_{i, t} | q, o_{i,<t}\right)}$.
However, by constraining the optimized policy $\pi_\theta$ toward the behavior policy $\pi_\beta$ through its clipping operation, $\mathcal{J}_{\text{off-}\pi_\beta}$ introduces three issues. (i) Suboptimal constraint: $\pi_{\beta}$ is usually inferior to the current policy, so constraining toward it may impair performance. (ii) Unstable updates: Unlike on-policy updates that consistently constrain $\pi_\theta$ toward $\pi_{\theta_{\text{old}}}$, this objective enforces proximity to different behavior policies across replayed samples, disrupting the standard trust-region dynamics and undermining its stability advantage. (iii) Less informative signals: Since $\pi_\theta$ generally deviates much more substantially from $\pi_\beta$ than from $\pi_{\theta_{\text{old}}}$, the objective causes clipping to occur more frequently, thereby reducing the informativeness of the gradient signal.

These issues stem from using $\pi_\beta$ as the proximal policy that constrains the optimized policy. To address them, we aim to retain $\pi_{\theta_{\text{old}}}$ as the proximal policy in the clipping ratio, even when optimizing over off-policy samples drawn from different $\pi_\beta$.
To this end, we adopt a decoupled importance sampling scheme that factorizes the correction term $\rho_{i, t}(\theta,\beta)$ into the product of two importance weights, $\rho_{i, t}(\theta,\theta_{\text{old}})$ and $\rho_{i, t}(\theta_{\text{old}},\beta)$, and optimize the following decoupled off-policy objective:
\begin{equation}
\small
\mathcal{J}_{\text{off-dec}}(\theta)  = \underset{q, o_i \sim \pi_{\beta}}{\mathbb{E}} \Bigg[
\frac{1}{k}\sum_{i=1}^k \frac{1}{|o_i|}\sum_{t=1}^{|o_i|} \rho_{i, t}(\theta_{\text{old}},\beta)  \min \Big( \rho_{i, t}(\theta,\theta_{\text{old}}) \hat{A}_i, \clip(\rho_{i, t}(\theta,\theta_{\text{old}}), 1 \pm \epsilon) \hat{A}_i \Big) \Bigg].
\label{eq:J_off-token}
\end{equation}
The objective $\mathcal{J}_{\text{off-dec}}$ preserves the same off-policy correction as $\mathcal{J}_{\text{off-}\pi_\beta}$ while modifying only the trust-region constraint. The following lemma formalizes this connection.
\begin{lemma}
\label{lem:decoupled}
$\mathcal{J}_{\text{off-dec}}$ and $\mathcal{J}_{\text{off-}\pi_\beta}$ differ only in their trust-region constraints:
the former clips $\rho_{i,t}(\theta,\theta_{\text{old}})$, whereas the latter clips $\rho_{i,t}(\theta,\beta)$.
For any token at which neither ratio is clipped, the two objectives yield identical values and gradients.
\end{lemma}

This formulation decouples the behavior policy that generates the data from the proximal policy that constrains the update, enabling POPO to maintain a consistent, high-quality trust-region constraint across both on-policy and off-policy data, and thereby enhance training stability. While the decoupled formulation originates from classical RL for batch-size-invariant optimization~\citep{hilton2022batch}, and has recently been adopted in LLM training to address precision mismatch between rollout and training engines~\citep{zheng2025stabilizing,fu2025areal}, this work repurposes it for replay-based off-policy optimization with stable trust-region updates.

\subsection{Overall Algorithm}
This section summarizes our final algorithm, \algo.
As described in \cref{sec:prioritized group replay}, the final training batch $\mathcal{G}_{\text{train}}$ consists of the online effective group set $\mathcal{G}_{\text{on-eff}}$ and the off-policy effective group set $\mathcal{G}_{\text{off-eff}}$. For samples in $\mathcal{G}_{\text{on-eff}}$, we perform standard GRPO training by Eq.~\eqref{eq:GRPO}, while for samples in $\mathcal{G}_{\text{off-eff}}$, we apply decoupled off-policy optimization using Eq.~\eqref{eq:J_off-token}. In practice, this only requires multiplying the standard policy loss of each replayed sample by $\rho_{i, t}(\theta_{\text{old}},\beta)$ before loss aggregation.
Therefore, \algo is easy to implement and introduces neither additional computational overhead nor modeling complexity compared to standard RLVR methods.
The complete algorithm is presented in \cref{algo}, and an overview of the framework is illustrated in Fig.~\ref{fig:framework}.

\begin{algorithm}[tb]
  \caption{Group Prioritized Off-Policy Optimization}
  \label{algo}
  \begin{algorithmic}
    \STATE {\bfseries Input:} Prompt dataset $\mathcal{D}$; Batch size $B$; Group size $k$; LLM $\pi_{\theta}$; Total training steps $T$.
    \STATE Initilize a FIFO replay buffer $\mathcal{B} = \varnothing$ with capacity $B$;
    \FOR{$t=1$ {\bfseries to} $T$}

    \STATE Sample a batch of prompts $\mathcal{Q}_{\text{on}}$ from $\mathcal{D}$\;
    \STATE Generate $k$ responses per $q \in \mathcal{Q}_{\text{on}}$ using $\pi_{\theta}$ to obtain the online group set $\mathcal{G}_{\text{on}} = \left\{ g^q \middle| q \in \mathcal{Q}_{\text{on}} \right\}$\;
    \STATE Remove group $g^q$ from $\mathcal{G}_{\text{on}}$ if $\text{std}\left(\{r(q,o_i)\}_{i=1}^k\right) = 0$ to obtain the online effective set $\mathcal{G}_{\text{on-eff}}$\;
    \STATE Construct the off-policy effective set $\mathcal{G}_{\text{off-eff}} = \left\{g_{|\mathcal{B}|-i+1} \in \mathcal{B}\;\middle|\;i = 1, \ldots, B - |\mathcal{G}_{\text{on-eff}}|\right\}$\;
    \STATE Update the LLM $\pi_{\theta}$ using groups from $\mathcal{G}_{\text{on-eff}}$ and $\mathcal{G}_{\text{off-eff}}$ by Eq.~\eqref{eq:GRPO} and \eqref{eq:J_off-token}, respectively\;
    \STATE Add groups $g^q \in \mathcal{G}_{\text{on-eff}}$ to the FIFO replay buffer $\mathcal{B}$\;
    \ENDFOR
  \end{algorithmic}
\end{algorithm}

\section{Experiments}
\label{sec:experiments}
In this section, we conduct a series of experiments to examine the effectiveness of \algo. Appendices \ref{appsec:implementation}, \ref{appsec:results}, and \ref{appsec:dataexample} provide experimental details, extended results, and data examples, respectively.

\subsection{Experimental Setup}
We evaluate POPO across three challenging reasoning domains, using diverse model sizes and architectures. Performance is measured by average pass@1 over multiple rollouts per problem.

\textbf{Mathematics.}
Following prior work~\citep{luo2025deepscaler,qu2025can}, we train DeepSeek-R1-Distill-Qwen-1.5B (DSR-1.5B) and 7B (DSR-7B)~\citep{guo2025deepseek} on the DeepScaleR dataset~\citep{luo2025deepscaler}, which comprises 40k competition-level math problems. Evaluation is conducted on a suite of standard mathematical benchmarks: AIME25, AMC23, MATH500~\citep{lightman2023let}, Minerva Math (Minerva.)~\citep{lewkowycz2022solving}, and OlympiadBench (Olymp.)~\citep{he2024olympiadbench}. 
To examine out-of-distribution (OOD) generalization, we further use general reasoning benchmarks including MMLU-Pro~\citep{wang2024mmlu}, ARC-c~\citep{clark2018think}, and GPQA-diamond (GPQA)~\citep{rein2024gpqa}.

\textbf{Numerical Planning.}
We adopt the Countdown Numbers Game to evaluate arithmetic planning. Following prior work~\cite{chen2025self,qu2025can}, we train Qwen2.5-3B~\citep{yang2024qwen2} on a 20k subset of the Countdown-34 (CD-34) dataset~\citep{tinyzero}, and evaluate on its held-out split and a harder variant, Countdown-4 (CD-4).

\textbf{Visual Geometry.}
Following prior work~\citep{qu2025can}, we train the vision-language model Qwen2.5-VL-3B-Instruct~\citep{bai2025qwen2} on the Geometry3k training set~\citep{lu2021inter, geometry3k_dataset}, and evaluate on its test split. The dataset contains 3k visual-geometric problems that require both image understanding and symbolic reasoning.

\textbf{Baselines.}
We compare against the following representative baselines.
(i) GRPO: the GRPO algorithm with the clip higher technique~\citep{yu2025dapo}. (ii) DAPO: a resource-intensive approach that oversamples and filters prompts using rollout feedback to form fully effective online batches~\citep{yu2025dapo}. (iii) MoPPS: a predictive sampling method that predicts effective prompts using Bayesian inference~\citep{qu2025can}. (iv) ARPO: a trajectory replay method that reuses historical correct responses when online groups contain only incorrect ones~\citep{lu2025arpo}.
On DSR-1.5B, we additionally include the off-policy baselines ReMix~\citep{liang2025squeeze}, REPO~\citep{li2025repo}, RR~\citep{sun2025improving}, and the predictive sampling method GRESO~\citep{zheng2025act}. Our primary goal is reducing computational overhead relative to DAPO, rather than surpassing its accuracy.

\begin{figure}[t]
    \centering
    \includegraphics[width=\linewidth]{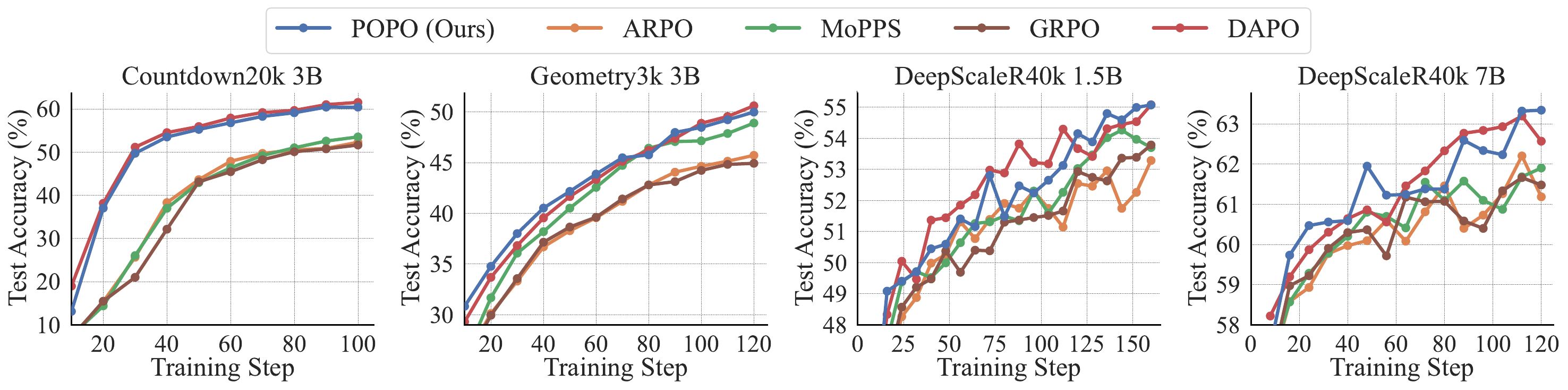}
    \vspace{-15pt}
    \caption{
    Training curves of different methods across three reasoning tasks with varying model sizes. 
    Notably, DAPO is included as a high-resource baseline as it requires substantially more LLM rollouts per training step than competing methods.
    }
    \vspace{-7pt}
    \label{fig:performance}
\end{figure}

\begin{figure}[t]
    \centering
    \includegraphics[width=\linewidth]{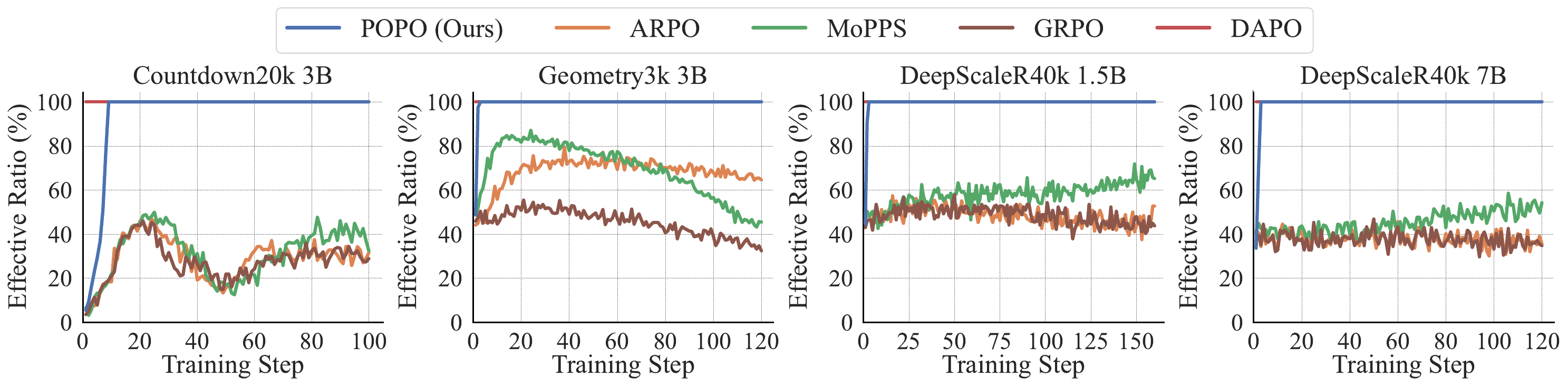}
    \vspace{-15pt}
    \caption{
    Proportion of effective samples within training batches across different RLVR methods.
    }
    \vspace{-7pt}
    \label{fig:effective ratio}
\end{figure}

\subsection{RLVR Performance and Efficiency}

\textbf{Training Progress.}
Fig.~\ref{fig:performance} presents training curves across tasks and models, plotting test accuracy against training steps. \algo consistently yields faster improvement and higher final accuracy than GRPO, ARPO, and MoPPS, while performing on par with the computationally expensive baseline DAPO. \algo also reaches GRPO's final performance with only 40\%–70\% of training steps.
To understand how training data quality contributes to this gap, Fig.~\ref{fig:effective ratio} reports the effective sample ratio (i.e., the proportion of effective data in training batches) of each method throughout training. 
While \algo and DAPO optimize full batches of effective samples, GRPO typically maintains effective ratios below 50\% across tasks. 
Consistent with our analysis in \cref{sec:ineffective data}, MoPPS and ARPO improve the effective ratio mainly on small datasets: on Geometry (3k), their effective ratios are markedly higher; this advantage diminishes on larger datasets such as DeepScaleR (40k) and Countdown (20k), where limited per-prompt history weakens their predictive or replay mechanisms, causing performance similar to GRPO. Notably, although ARPO achieves higher effective ratios on Geometry, its performance gain is limited, possibly due to inaccurate advantage estimation.

\begin{table*}[t]
  \centering
  \caption{
Evaluation on mathematics benchmarks.
‘+’ denotes finetuning with the method.
}
\vspace{-5pt}
  \resizebox{\linewidth}{!}{
  \setlength{\tabcolsep}{4pt}
    \begin{tabular}{rcccccc|cccc|cc}
    \toprule
    \multicolumn{1}{c}{\multirow{3}{*}{Method}} & \multicolumn{6}{c|}{\textbf{In-Distribution}}          & \multicolumn{4}{c|}{\textbf{Out-of-Distribution}} &\\
\cmidrule{2-11}          & \small{AIME25} & \small{AMC23} & \small{Minerva.} & \small{MATH500} & \small{Olymp.} & \small{Avg.} & \small{MMLU-P} & \small{ARC-c} & \small{GPQA}  & \small{Avg.} & \small{Runtime}
\\
\midrule
DSR-1.5B & 19.2 & 57.0 & 23.5 & 78.0 & 38.2 & 43.2 & 21.1 & 42.2 & 23.0 & 28.8 & - \\
+GRPO & 26.9 & 76.1 & 30.8 & 86.2 & 49.0 & 53.8 & 22.3 & 45.4 & 27.5 & 31.7 & 16h \\
+ARPO  & 27.7 & 76.7 & 28.8 & 86.2 & 48.1 & 53.3 & 22.3 & 45.3 & 26.3 & 31.3 & 16h \\
+GRESO & 25.6 & 77.1 & 30.7 & 86.2 & 50.6 & 54.1 & 24.7 & 46.7 & 26.4 & 32.6 & 27h \\
+ReMix & 27.4 & 79.4 & 30.0 & 85.4 & 51.0 & 54.5 & 22.5 & 45.3 & 27.7 & 31.8 & 16h \\
+RR & 26.5 & 76.7 & 31.0 & 85.0 & 49.6 & 53.8 & 21.2 & 45.4 & 27.1 & 31.2 & 14h \\
+REPO & 28.3 & 77.5 & 30.3 & 86.4 & 50.9 & 54.6 & 21.9 & 44.9 & 27.4 & 31.4 & 19h \\
+MoPPS & 28.4 & 77.7 & 29.9 & 85.2 & 50.1 & 54.3 & 21.4 & 44.6 & 27.5 & 31.2 & 17h \\
+DAPO & 27.5 & 79.2 & 30.4 & 87.2 & 51.0 & \textbf{55.1} & 24.5 & 46.7 & 26.8 & \underline{32.7} & 30h \\
\rowcolor{myhighlight}
+\algo & 29.9 & 77.9 & 29.9 & 86.2 & 51.5 & \textbf{55.1} & 22.6 & 45.9 & 30.3 & \textbf{32.9} &16h \\
\midrule
DSR-7B & 30.2 & 76.5 & 35.8 & 87.2 & 47.5 & 55.4 & 50.5 & 76.1 & 19.2 & 48.6 & - \\
+GRPO & 35.2 & 87.7 & 38.1 & 92.4 & 54.9 & 61.7 & 50.9 & 75.6 & 23.8 & 50.1 & 31h \\
+ARPO & 36.4 & 87.5 & 38.6 & 91.4 & 57.1 & 62.2 & 51.6 & 75.7 & 24.4 & 50.6 & 31h \\
+MoPPS & 33.3 & 88.0 & 39.2 & 92.2 & 56.8 & 61.9 & 51.6 & 76.7 & 25.4 & \textbf{51.2} & 32h \\
+DAPO & 39.7 & 89.1 & 37.9 & 92.2 & 57.1 & \underline{63.2} & 52.0 & 76.9 & 23.4 & 50.8 & 55h \\
\rowcolor{myhighlight}
+\algo & 39.4 & 88.9 & 38.0 & 92.6 & 57.7 & \textbf{63.3} & 52.1 & 76.3 & 25.2 & \textbf{51.2} & 34h \\
\bottomrule
    \end{tabular}
    }
  \label{tab:matheval}
\end{table*}

\begin{table}[htbp]
\vspace{-3mm}
\centering
\caption{Evaluation on Countdown and Geometry. ‘+’ denotes finetuning with the method.}
\resizebox{0.98\linewidth}{!}{
\setlength{\tabcolsep}{6.6pt}
\begin{tabular}{rccccc|rccc}
\toprule
\multicolumn{1}{c}{\multirow{3}{*}{Method}} & \multicolumn{5}{c|}{Countdown} & \multicolumn{1}{c}{\multirow{3}{*}{Method}} & \multicolumn{3}{c}{Geometry} \\
\cmidrule(lr){2-6} \cmidrule(lr){8-10}
& CD4 & CD34 & Avg. & Runtime & Rollouts
&  & Avg. & Runtime & Rollouts \\
\midrule
Qwen-3B
& 0.5 & 2.6 & 1.5 & - & -
& Qwen-VL-3B & 24.6 & - & - \\
+GRPO 
& 39.0 & 64.3 & 51.6 & 2.8h & \textbf{205k}
& +GRPO & 44.9 & 6.8h & \textbf{492k} \\
+ARPO 
& 39.7 & 64.8 & 52.3 & 2.9h & \textbf{205k}
& +ARPO & 45.7 & 6.5h & \textbf{492k} \\
+MoPPS 
& 40.7 & 66.3 & 53.5 & 3.4h & \textbf{205k}
& +MoPPS & 48.9 & 7.4h & \textbf{492k} \\
+DAPO 
& 49.8 & 73.2 & \textbf{61.5} & 5.6h & 877k
& +DAPO & \textbf{50.6} & 11.2h & 1438k \\
\rowcolor{myhighlight}
+POPO 
& 48.7 & 72.1 & \underline{60.4} & 3.2h & \textbf{205k}
& +POPO & \underline{50.0} & 6.8h & \textbf{492k} \\
\bottomrule
\end{tabular}
}
\label{tab:cd geo eval}
\end{table}

\textbf{Generalization Performance.}
We evaluate the trained models across multiple challenging benchmarks to assess their generalization capabilities. For models trained on DeepScaleR, \cref{tab:matheval} reports results on both standard in-distribution mathematical benchmarks and OOD general reasoning benchmarks. \cref{tab:cd geo eval} presents evaluations on Countdown and Geometry. Across all benchmarks, \algo consistently matches the performance of DAPO while outperforming the other baselines.

\begin{wrapfigure}{r}{7cm}
    \vspace{-13.5pt}
    \centering
    \includegraphics[width=\linewidth]{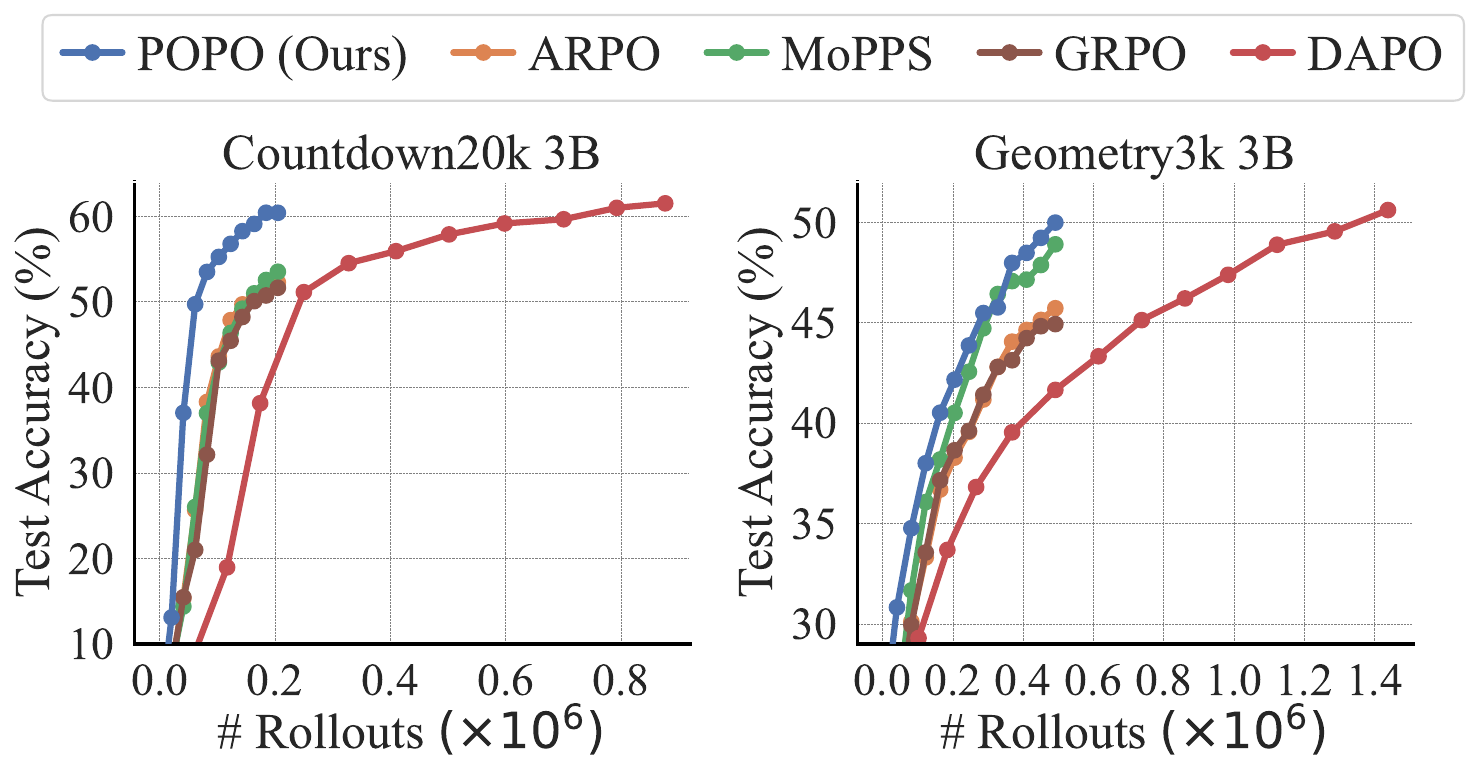}
    \vspace{-5.3mm}
    \caption{
    Training curves of different methods as a function of the number of LLM rollouts.
    }
    \label{fig:performance_rollout_2fig}
    \vspace{-10pt}
\end{wrapfigure}

\textbf{Rollout and Runtime Efficiency.}
We further compare methods in terms of rollout usage and runtime. \cref{tab:matheval,tab:cd geo eval} list the total number of generated rollouts and runtime for each task, and Fig.~\ref{fig:performance_rollout_2fig} plots model performance against rollout count during training (see Fig.~\ref{fig:performance_rollout} for full curves). \algo achieves competitive performance with DAPO using only 30\% of its rollout budget and roughly half the runtime. The slight runtime increase of \algo over GRPO stems from longer response lengths associated with higher performance, as shown in Fig.~\ref{fig:responselength}.

\subsection{Effects of Components in \algo}
\label{sec:ablation}

\textbf{Replay Mechanisms.}
We examine alternative replay strategies. POPO-KL selects replay groups by ranking KL divergence from a buffer 5$\times$ larger than the batch. As shown in Fig.~\ref{fig:abla} (left), POPO achieves similar performance but with significantly better efficiency (3.2h vs 4.8h), suggesting recency as an effective low-cost proxy for distance-based replay. GRPO-filter removes ineffective groups and updates only on on-policy effective groups. This filtering reduces the training batch size, which increases gradient magnitude due to batch averaging but also amplifies noise. As shown in Fig.~\ref{fig:abla} (left), this variant is less effective than POPO and only marginally outperforms GRPO, highlighting the importance of replay-based batch completion.

POPO's prioritized replay taking into account both sample quality and off-policiness. To disentangle their effects, we test the following variants: \algo-ineff substitutes ineffective online groups by replaying all recent groups regardless of their effectiveness, whereas \algo-stale(n) randomly replays historical effective groups from a larger buffer of size $n\times B$, regardless of their staleness. As shown in Fig.~\ref{fig:abla} (middle), \algo-ineff causes substantial performance degradation, validating the quality criterion. Furthermore, while \algo-stale(1) performs slightly worse, \algo-stale(10) collapses due to a large off-policy gap, underscoring the necessity of the off-policiness criterion.

\textbf{Off-policy Optimization Objectives.}
To examine different off-policy optimization objectives, we fix the replay mechanism in \algo and test $\mathcal{J}_{\text{off-}\pi_{\text{old}}}$ and $\mathcal{J}_{\text{off-}\pi_\beta}$ as alternative off-policy objectives, denoting the resulting variants as \algo-$\pi_{\text{old}}$ and \algo-$\pi_\beta$. As shown in Fig.~\ref{fig:abla} (right), \algo-$\pi_\beta$ leads to a slowdown in performance improvements due to suboptimal trust-region constraints, while \algo-$\pi_{\text{old}}$ experiences performance collapse, possibly due to the bias in the optimization process. These results support the effectiveness of the decoupled off-policy optimization objective $\mathcal{J}_{\text{off-dec}}$.

\subsection{Additional Analysis}
\begin{figure}[t!]
    \centering
    \includegraphics[width=0.93\linewidth]{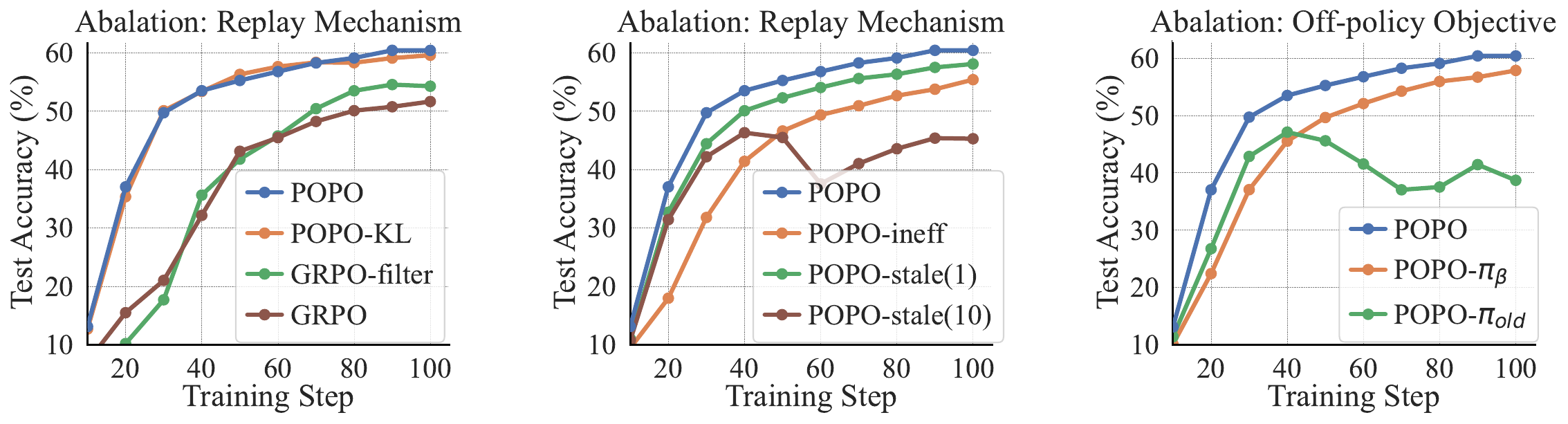}
    \vspace{-2mm}
    \caption{
    Ablation study on Countdown, including variants that use KL-based replay (\algo-KL), optimize only over on-policy effective groups (GRPO-filter), replay all recent groups regardless of effectiveness (\algo-ineff), replay stale effective groups (\algo-stale(n)), and employ alternative off-policy optimization objectives, $\mathcal{J}_{\text{off-}\pi_{\text{old}}}$ (\algo-$\pi_{\text{old}}$) and $\mathcal{J}_{\text{off-}\pi_\beta}$ (\algo-$\pi_\beta$).
    }
    \label{fig:abla}
    \vspace{-5pt}
\end{figure}
\begin{figure}[t!]
\setlength{\subfigcapskip}{-6pt}
    \vspace{-1mm}
    \centering
    \hspace{-3mm}
    \subfigure[]{
    \label{fig:replenish ratio}
    \includegraphics[width=0.253\linewidth]{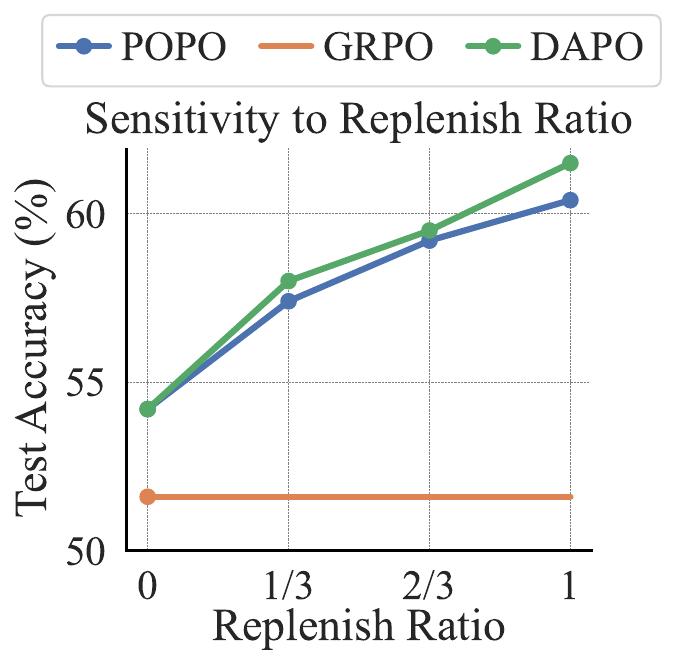}
    }
    \hspace{-5mm}
    \subfigure[]{
    \label{fig:group size}
    \includegraphics[width=0.253\linewidth]{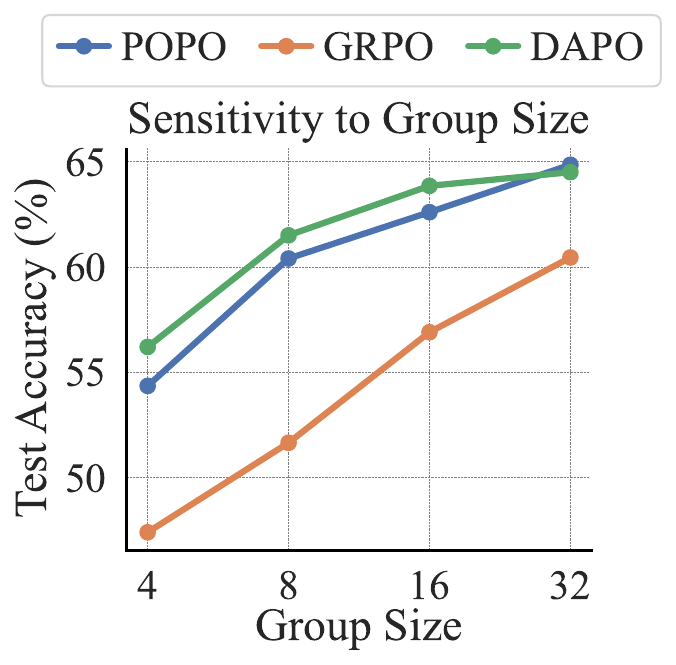}
    }
    \hspace{-5mm}
    \subfigure[]{
    \label{fig:rlvr}
    \includegraphics[width=0.251\linewidth]{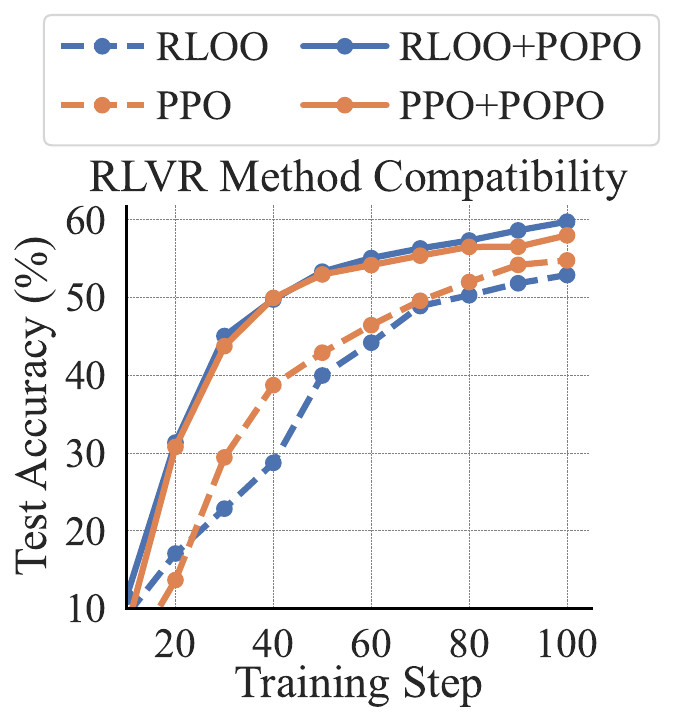}
    }
    \hspace{-5mm}
    \subfigure[]{
    \label{fig:sampling}
    \includegraphics[width=0.251\linewidth]{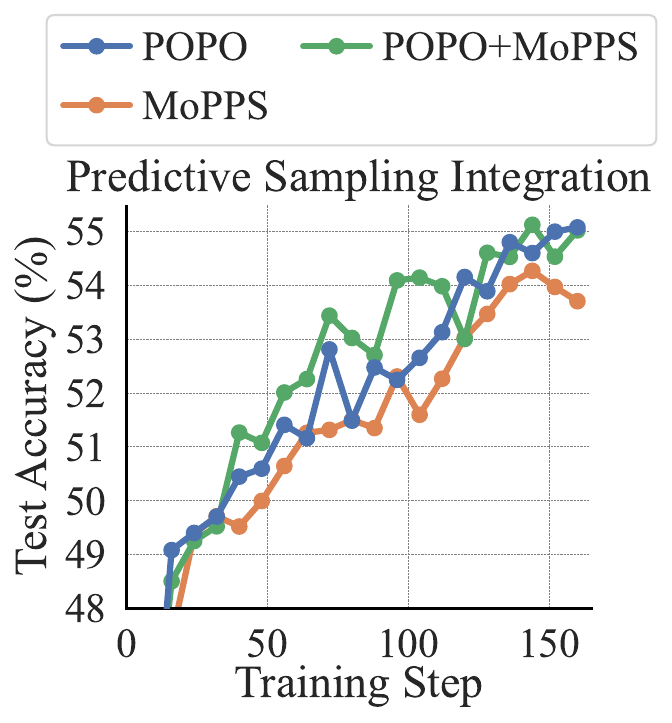}
    }
    \hspace{-3mm}
    \vspace{-3mm}
    \caption{(a) Final performance of \algo and DAPO on Countdown under varying batch replenishment ratios. (b) Final performance of \algo, GRPO, and DAPO on Countdown under varying response group sizes. (c) Training curves on Countdown when applying \algo to RLOO and PPO. (d) Training curves on DeepScaleR 1.5B when integrating \algo with predictive sampling (MoPPS).
    }
    \label{fig:additional analysis}
    \vspace{-2mm}
\end{figure}

\textbf{Sensitivity to Batch Replenishment Ratio.}
POPO and DAPO address ineffective samples via different strategies: After filtering ineffective groups, POPO replenishes the batch with off-policy groups, whereas DAPO generates additional on-policy rollouts. To systematically compare these strategies, we vary the replenishment ratio and report results on Countdown. As shown in Fig.~\ref{fig:replenish ratio}, performance increases monotonically as the proportion of effective data grows. Notably, POPO consistently matches DAPO in accuracy while maintaining a constant rollout cost.

\textbf{Sensitivity to Group Size.}
We evaluate \algo, GRPO, and DAPO under varying response group sizes, specifically $k \in \{4, 8, 16, 32\}$. Fig.~\ref{fig:group size} presents the final average test accuracies for different methods on Countdown. The results show that \algo consistently outperforms GRPO and performs comparably to DAPO across these group sizes, demonstrating the generality of \algo.

\textbf{Compatibility with Base RLVR Methods.}
While our primary experiments focus on GRPO, \algo can be readily integrated into other base RLVR algorithms. We combine \algo with RLOO~\citep{ahmadian2024back} and PPO~\citep{schulman2017proximal}, and test on the Countdown task. As shown in Fig.~\ref{fig:rlvr}, incorporating \algo consistently improves both training efficiency and final accuracy across these methods. These findings highlight its broad applicability for improving sample efficiency in RLVR pipelines.

\textbf{Integration with Sampling Methods.}
\algo is orthogonal to active prompt sampling methods and can be naturally integrated with them (e.g., MoPPS), potentially improving online data utilization. We implement POPO+MoPPS and compare it with standalone methods in Fig.~\ref{fig:sampling}, with full results, including batch composition ratios, shown in Fig.~\ref{appfig:sampling}. Compared to \algo alone, the combination yields a higher ratio of online effective groups and leads to even faster policy improvement, indicating complementary benefits between the two strategies.

\section{Conclusion and Limitations}
\label{sec:conclusion}
This work presents POPO, a simple yet effective framework for improving sample efficiency in RL finetuning of LLMs. By replacing ineffective on-policy groups with high-quality off-policy groups through recency-based replay, POPO constructs training batches composed entirely of effective samples without incurring additional rollout cost. POPO further employs decoupled importance sampling to correct off-policy bias while maintaining a consistent trust-region constraint for stable optimization. Empirical results across diverse reasoning tasks show that POPO achieves performance comparable to the SOTA method while requiring substantially fewer rollouts. 

While POPO achieves promising results, several limitations remain. For example, the replay mechanism does not necessarily select the most informative samples. Exploring hybrid strategies that combine short-term buffers with a long-term reservoir is promising. Moreover, future work may further reduce rollout costs by generating fewer online samples and leveraging more off-policy data.

\bibliographystyle{plainnat}
\bibliography{neurips_2026.bib}

\clearpage

\appendix

\newpage

\section{Related Work}
\label{sec:related work}
\paragraph{RL Finetuning for LLMs.}
RL has emerged as a central paradigm for adapting LLMs to complex tasks and target behaviors. Among its variants, Reinforcement Learning with Human Feedback (RLHF) has demonstrated strong effectiveness in aligning LLM outputs with human preferences and safety requirements~\citep{ouyang2022training,dong2024rlhf,rafailov2023direct,dai2023saferlhf,sun2023aligningrlhf,sheng2024hybridflow}. In settings where reward signals are verifiable, such as mathematical problem solving or program synthesis, Reinforcement Learning with Verifiable Rewards (RLVR) has led to substantial improvements in LLM reasoning performance~\citep{jaech2024openai,shao2024deepseekmath,team2025kimi,chu2025sftrl,guo2025deepseek}.
From an algorithmic standpoint, Proximal Policy Optimization (PPO)~\citep{schulman2017proximal}, a widely used policy-gradient method, can be directly applied to LLM finetuning. Building on this foundation, Group Relative Policy Optimization (GRPO)~\citep{shao2024deepseekmath} removes the need for a separate value network by employing a group-normalized advantage estimator, significantly reducing computational overhead. As a result, GRPO has quickly become a prevalent choice for RLVR. Subsequent research has further refined RLVR by addressing gradient bias, improving training stability, and enhancing computational efficiency~\citep{yuan2025vcppo,yue2025vapo,liu2025understanding,yu2025dapo,kazemnejad2024vineppo,hu2025reinforce++,zheng2025group}.
On the application front, recent efforts have extended RL finetuning to a wider range of tasks and increasingly large-scale models~\citep{luo2025deepscaler,dang2025reinforcement,luo2025deepcoder,zeng2025simplerl,meng2025mm,xu2024llava}. In parallel, infrastructure-level advances have produced scalable and resource-efficient frameworks that support distributed RL training tailored to the demands of LLMs~\citep{sheng2024hybridflow,hu2025open}.

\paragraph{Active Sampling for RL Finetuning.}

A growing body of work emphasizes that the effectiveness of RL finetuning is severely constrained by ineffective training samples~\citep{yu2025dapo}, which are either too easy or too difficult, yielding degenerate rewards and limited optimization signals~\citep{bae2025online, chen2025self}.
To address this, recent studies explore active prompt sampling strategies that dynamically select informative prompts on a per-step or per-epoch basis~\citep{bae2025online, zhang2025srpo, cui2025process, liu2025prorl}. 
A representative SOTA method, DAPO~\citep{yu2025dapo}, oversamples and filters prompts using rollout feedback, improving sample quality but incurring substantial overhead from enlarged batch generation.
To avoid this cost, predictive sampling methods aim to identify prompts with intermediate success probability prior to rollout~\citep{qu2025can,shen2025bots,xu2025single,zheng2025act}. 
These methods typically estimate the current success rate of each prompt based on historical reward statistics, employing Bayesian inference or heuristic modeling.
For example, GRESO~\citep{zheng2025act} proposes selective rollouts to reduce unnecessary generations, while PCL~\citep{gao2025prompt} estimates prompt difficulty for prioritized sampling.
Nevertheless, accurate prediction is difficult under continually evolving policies and sparse observations, often limiting the number of effective samples.
In contrast, \algo avoids additional rollout cost by reusing previously collected effective groups, enabling full batches of effective data for optimization and offering an orthogonal perspective focused on post-hoc utilization rather than pre-rollout filtering.

\paragraph{Experience Replay for RL Finetuning.}
Experience replay is an important technique in classical RL to improve sample efficiency~\citep{lin1992self,mnih2013playing,mnih2015human,schaul2015prioritized} and has recently gained traction in RL finetuning due to the high online generation cost of LRMs~\citep{zhang2025survey}.
Recent studies show that replaying certain trajectories can accelerate training. 
For example, ARPO~\citep{lu2025arpo} and AR3PO~\citep{zhang2025improving} store historical successful trajectories and substitute one response when all online rollouts in a group fail. RLEP~\citep{zhang2025rlep} reuses successful trajectories from a well-trained policy and re-trains a policy by appending successful trajectories to online groups. ExGRPO~\citep{zhan2025exgrpo} identifies rollout correctness and entropy as indicators of experience value, and mixes past correct responses with the lowest entropy into online groups. RRL~\citep{dou2025improving} selectively replays promising earlier reasoning traces to continue exploring them. DEPO~\citep{tang2025towards} introduces a two-stage filtering and replay strategy with offline data selection and online explorability-based filtering to improve data efficiency. RR~\citep{sun2025improving} generates fewer online rollouts and reuses past ones to fill a fixed fraction of batches, thereby reducing rollout cost. ReMix~\citep{liang2025squeeze} revisits the generalized proximal gradient theory~\citep{queeney2021generalized} and introduces mix-policy proximal policy gradient in RLVR.
However, these methods do not fully resolve the issue of ineffective samples. 
Moreover, their off-policy optimization objectives may be subject to issues such as intractable bias, suboptimal constraints, or invalid advantage estimates.
In contrast, \algo addresses these limitations and leverages high-quality off-policy optimization signals to accelerate RL finetuning.

\section{Proofs}
\label{appsec:proofs}

\begin{lemma}
Under proximal policy updates, assume that $1-\epsilon \leq \frac{\pi_{i+1}(a|s)}{\pi_{i}(a|s)} \leq 1+\epsilon$ for all steps $i$ and state-action pairs $(s,a)$. Then, for any step difference $n\ge 1$, $\mathrm{TV}\big(\pi_{i}(\cdot|s),\pi_{i+n}(\cdot|s)\big) \le \frac{\epsilon n}{2}$.
\end{lemma}

\begin{proof}
We prove the lemma for a fixed state $s$. For brevity, write $\pi_i(a)=\pi_i(a|s)$. Recall that the total variation distance between two distributions $p,q$ on the same finite space is
\begin{equation}
\mathrm{TV}(p,q)=\frac12\sum_a |p(a)-q(a)|.
\end{equation}
From the assumption,
\begin{equation}
1-\epsilon \le \frac{\pi_{i+1}(a)}{\pi_i(a)} \le 1+\epsilon,
\quad \forall a,
\end{equation}
which implies
\begin{equation}
\left|\frac{\pi_{i+1}(a)}{\pi_i(a)}-1\right| \le \epsilon,
\quad \forall a.
\end{equation}
Multiplying both sides by $\pi_i(a)\ge 0$ yields
\begin{equation}
|\pi_{i+1}(a)-\pi_i(a)|
= \pi_i(a)\left|\frac{\pi_{i+1}(a)}{\pi_i(a)}-1\right|
\le \epsilon \, \pi_i(a),
\quad \forall a.
\end{equation}
Summing over all actions $a$ and using the normalization $\sum_a \pi_i(a)=1$, we obtain
\begin{equation}
\sum_a |\pi_{i+1}(a)-\pi_i(a)|
\le \epsilon \sum_a \pi_i(a)
= \epsilon.
\end{equation}
Therefore,
\begin{equation}
\mathrm{TV}(\pi_i,\pi_{i+1})
=\frac12\sum_a |\pi_{i+1}(a)-\pi_i(a)|
\le \frac{\epsilon}{2}.
\label{appeq:one-step bound}
\end{equation}
Next, for any $n\ge 1$, we apply the triangle inequality for total variation distance repeatedly along the sequence of intermediate policies:
\begin{equation}
\mathrm{TV}(\pi_i,\pi_{i+n})
\le \sum_{k=0}^{n-1} \mathrm{TV}(\pi_{i+k},\pi_{i+k+1}).
\label{appeq:triangle}
\end{equation}
Applying the one-step bound in Eq.~\eqref{appeq:one-step bound} to each term in the sum,
\begin{equation}
\mathrm{TV}(\pi_{i+k},\pi_{i+k+1}) \le \frac{\epsilon}{2},
\quad \forall k\ge 0.
\end{equation}
Hence,
\begin{equation}
\mathrm{TV}(\pi_i,\pi_{i+n})
\le \sum_{k=0}^{n-1} \frac{\epsilon}{2}
= \frac{\epsilon n}{2}.
\end{equation}
Since the argument holds for an arbitrary fixed state $s$, we conclude that
\begin{equation}
\mathrm{TV}\big(\pi_i(\cdot|s),\pi_{i+n}(\cdot|s)\big)
\le \frac{\epsilon n}{2}
\end{equation}
for all step differences $n$, which completes the proof.
\end{proof}

\begin{lemma}
\label{applem:decoupled}
$\mathcal{J}_{\text{off-dec}}$ and $\mathcal{J}_{\text{off-}\pi_\beta}$ differ only in their trust-region constraints:
the former clips $\rho_{i,t}(\theta,\theta_{\text{old}})$, whereas the latter clips $\rho_{i,t}(\theta,\beta)$.
For any token at which neither ratio is clipped, the two objectives yield identical values and gradients.
\end{lemma}
\begin{proof}
Fix a token $(i,t)$ and define 
\begin{equation}
r = \rho_{i,t}(\theta,\theta_{\text{old}}), \qquad
w = \rho_{i,t}(\theta_{\text{old}},\beta), \qquad
\hat{A} = \hat{A}_i,
\end{equation}

Note that $\rho_{i,t}(\theta,\beta) = r w$. The per-token contributions of the two objectives $\mathcal{J}_{\text{off-}\pi_\beta}$ and $\mathcal{J}_{\text{off-dec}}$ are
\begin{equation}
L_\beta = \min\bigl(r w \hat{A},\; \operatorname{clip}(r w,\,1\pm\epsilon)\,\hat{A}\bigr), \qquad
L_{\text{dec}} = w \cdot \min\bigl(r \hat{A},\; \operatorname{clip}(r,\,1\pm\epsilon)\,\hat{A}\bigr).
\end{equation}

Thus $\mathcal{J}_{\text{off-}\pi_\beta}$ clips $rw$ while $\mathcal{J}_{\text{off-dec}}$ clips $r$, differing in the argument of the trust-region clipping.

Consider any token where neither ratio is clipped by its respective clipping operator.
This means
\begin{equation}
r \in [1-\epsilon,\,1+\epsilon] \quad\text{and}\quad r w \in [1-\epsilon,\,1+\epsilon].
\end{equation}
Under these conditions,
\begin{equation}
L_\beta = \min\big(r w \hat{A},\; r w \hat{A}\big) = w \cdot \min\big(r \hat{A},\; r \hat{A}\big) = L_{\text{dec}}.
\end{equation}
Thus $L_\beta = L_{\text{dec}}$ pointwise. 
Because the outer expectation over $q$ and $o_i\sim\pi_\beta$ is identical for both objectives,
their total values coincide on the set of unclipped tokens.
For gradients, the advantage $\hat{A}$ is treated as constant and $w = \rho_{i,t}(\theta_{\text{old}},\beta)$ is independent of $\theta$, so
\begin{equation}
\nabla_\theta L_\beta = \hat{A}\, \nabla_\theta(r w) = \hat{A}\, w\, \nabla_\theta r = \nabla_\theta L_{\text{dec}}.
\end{equation}
Hence on any unclipped token the two objectives coincide in value and gradient.
\end{proof}

\section{Discussions}
\label{appsec:discussions}
\paragraph{Behavior Policy under Selective Replay.}
POPO performs prioritized replay at the group level, which preserves a well-defined behavior policy for each prompt. Specifically, all trajectories within a group are generated by the same rollout policy (i.e., the LLM policy at that training step), ensuring that the replayed data follows a tractable behavior distribution. This property enables principled importance sampling for off-policy correction.
In contrast, prior trajectory-level replay methods perform selective replay at the level of individual trajectories within a prompt. Such filtering effectively samples from a truncated and implicit subset of the original behavior policy, resulting in an intractable distribution. Consequently, the corresponding importance correction becomes ill-defined, leading to unresolved bias. This distinction highlights a key advantage of group-level replay in maintaining both statistical correctness and optimization stability.

\paragraph{Sensitivity to Off-policy Gap and Robustness.}
Our ablations in Sec.~\ref{sec:ablation} show that excessive off-policy drift, induced by intentionally replaying stale groups (POPO-stale($n$)), leads to performance degradation. POPO is explicitly designed to mitigate this issue via (i) recency-based replay, which limits distribution shift, and (ii) decoupled importance sampling, which corrects residual bias while maintaining stable trust-region updates. As shown in Sec.~\ref{sec:ablation}, alternative replay strategies and objectives result in inferior performance, supporting the importance of these design choices. At the same time, POPO is not overly sensitive to moderate variations in recency: for example, POPO-stale(1) incurs only minor degradation and still significantly outperforms GRPO.

\paragraph{Robustness to Replay-induced Overfitting.}
POPO is designed to mitigate potential overfitting under replay and maintain strong generalization. Specifically, several mechanisms contribute to this robustness.
First, replay is bounded: the FIFO buffer has a small, fixed capacity and is continuously refreshed with newly collected effective groups, preventing the training data from collapsing to a static subset. Second, replay is used only to fill ineffective slots rather than replacing the entire batch, which preserves continuous exposure to fresh, randomly sampled prompts and responses. Third, as the policy model evolves, the set of effective samples also changes dynamically; it is not tied to a fixed subset of prompts. Combined with random online prompt sampling, this ensures continual exploration and discovery of new effective samples rather than repeatedly focusing on the same ones. Empirically, we do not observe signs of overfitting (e.g., performance degradation or instability); instead, POPO consistently improves both sample efficiency and generalization across tasks.

\paragraph{Sample Reuse and Off-policy Control.}
POPO adopts a small FIFO buffer with recency-based replay, which may discard earlier useful groups. This design is intentional to bound the off-policy gap: under proximal updates, recent policies are theoretically closer to the current policy (Lemma~\ref{lem:off-policiness}), making replayed samples more reliable for off-policy correction. Empirical results further support this choice. Our ablations show that replaying increasingly stale groups degrades performance, with POPO-stale(10) collapsing, while small recency windows remain stable.
This design reflects a deliberate trade-off between sample reuse and off-policy control, prioritizing stability and reliable optimization over the reuse of older data. Nevertheless, retaining certain informative historical groups may still be beneficial. Exploring hybrid strategies, such as combining short-term buffers with a small long-term reservoir to balance recency and diversity without significantly increasing off-policy gaps, is a promising direction for future work.

\paragraph{Variance Control in Importance Sampling.}
POPO effectively controls variance in importance sampling, even under policy drift, through two mechanisms. First, recency-based replay keeps behavior policies close to the current policy, directly reducing the magnitude of importance weights. Second, we apply clipping (with an upper bound of 2.0) to $\rho_{i,t}(\theta_{\text{old}}, \beta)$ to suppress outliers in practice. Empirically, we do not observe instability arising from variance, suggesting that these mechanisms are effective.

\paragraph{Generalization to Diverse Reward Structures.}
While Definition~\ref{def:effective} is motivated by common binary-reward RLVR settings, the underlying principle extends naturally to more general reward structures, including continuous or multi-level rewards. In such cases, the strict zero-variance condition can be relaxed to a small-variance threshold or other low-information criteria that similarly yield weak learning signals. Thus, the definition is task-specific, while the POPO framework itself does not fundamentally rely on strict binary rewards or exact zero variance; instead, it operates on the presence of informative relative differences across groups.

\section{Experimental Details}
\label{appsec:implementation}

\subsection{Details of Tasks and Models}
We evaluate \algo across three distinct and challenging reasoning domains: competition-level mathematics, numerical planning, and visual geometric reasoning. To assess its general applicability, conduct experiments using a diverse set of large language and multimodal models spanning different capacities and architectures. We fine-tune all models using the widely adopted GRPO algorithm implemented in the verl framework~\citep{sheng2024hybridflow}. Model performance is measured by test accuracy, reported as the average pass@1 over multiple independent completions per example, with the number of generations varying across benchmarks. The training datasets, evaluation benchmarks, and base models for each domain are described below, with illustrative data examples provided in \cref{appsec:dataexample}.

\subsubsection{Competition-level Mathematics}
\paragraph{Training Dataset.}

For competition-level mathematics, we train large reasoning models on the DeepScaleR dataset~\citep{luo2025deepscaler}, which consists of 40,315 problems designed to reflect competition-level difficulty.
Specifically, we use the Hugging Face release from \url{https://huggingface.co/datasets/agentica-org/DeepScaleR-Preview-Dataset}.

\paragraph{Test Benchmarks.}
We evaluate mathematical reasoning ability on a diverse collection of benchmarks, including AIME24, AMC23, MATH500~\citep{lightman2023let}, Minerva Math~\citep{lewkowycz2022solving}, and OlympiadBench~\citep{he2024olympiadbench}, using the datasets hosted at \url{https://huggingface.co/datasets/math-ai}.
Training curves summarize results as the average accuracy across these mathematics benchmarks.
To examine generalization beyond the training distribution, we further conduct evaluations on three broader reasoning benchmarks: MMLU-Pro~\citep{wang2024mmlu}, ARC-c~\citep{clark2018think}, and GPQA-diamond (GPQA)~\citep{rein2024gpqa}, using the datasets provided by \citet{yan2025learning}. These benchmarks contain general reasoning problems in domains such as science and business.
Following prior work~\citep{gao2025prompt}, we adopt mixed Avg@k settings for evaluation, using larger $k$ for smaller datasets to reduce variance: Avg@32 for AIME25 and AMC23, Avg@4 for Minerva Math, Avg@1 for MATH500 and OlympiadBench, and Avg@8 for OOD benchmarks (MMLU-Pro, ARC-c, and GPQA-diamond).

\paragraph{Base Models.}
Following prior work~\citep{luo2025deepscaler}, we adopt two base models from DeepSeek~\citep{guo2025deepseek}: DeepSeek-R1-Distill-Qwen-1.5B, available at \url{https://huggingface.co/deepseek-ai/DeepSeek-R1-Distill-Qwen-1.5B} and DeepSeek-R1-Distill-Qwen-7B, available at \url{https://huggingface.co/deepseek-ai/DeepSeek-R1-Distill-Qwen-7B}.

\subsubsection{Numerical Planning}
\paragraph{Training Dataset.}
For arithmetic planning, we use the Countdown Number Game, in which agents must form a target value using basic operations over a given set of numbers~\citep{tinyzero}. Training is conducted on a 20,000-instance subset of the full Countdown-34 dataset, available from the Hugging Face repository at \url{https://huggingface.co/datasets/Jiayi-Pan/Countdown-Tasks-3to4}.

\paragraph{Test Benchmarks.}
Models are evaluated on two benchmarks: (i) CD-34, which contains 512 held-out instances from Countdown-34; and (ii) CD-4, comprising 512 problems from Countdown-4, a more challenging variant that uses four input numbers, available at \url{https://huggingface.co/datasets/Jiayi-Pan/Countdown-Tasks-4}.
Compared with CD-34, CD-4 substantially expands the search space and increases task difficulty. We report Avg@16 for both benchmarks, and training curves show the average test performance across CD-34 and CD-4.

\paragraph{Base Models.}
Following prior work~\cite{chen2025self,qu2025can}, we finetune the base model Qwen2.5-3B from Qwen~\citep{yang2024qwen2}, availiable at \url{https://huggingface.co/Qwen/Qwen2.5-3B}.

\subsubsection{Visual Geometry}
\paragraph{Training Dataset.}
For visual geometry, we use the training split of the Geometry3k dataset~\citep{lu2021inter, geometry3k_dataset}, available at \url{https://huggingface.co/datasets/hiyouga/geometry3k}. This dataset contains 2,101 diagram-based geometry problems that require both visual interpretation and symbolic reasoning.

\paragraph{Test Benchmarks.}
We assess the trained models on the corresponding benchmark test split, which comprises 601 visual reasoning problems, and report Avg@16.

\paragraph{Base Models.}
For visual geometric reasoning, we follow \citet{qu2025can} and employ the vision-language model Qwen2.5-VL-3B-Instruct from Qwen~\citep{bai2025qwen2}, available at \url{https://huggingface.co/Qwen/Qwen2.5-VL-3B-Instruct}.

\subsection{Training Details}

In the main experiments, both our approach and all baselines are built upon GRPO~\citep{shao2024deepseekmath} as the underlying RLVR algorithm. To ensure a fair comparison, we adopt the same GRPO implementation provided in the verl framework~\citep{sheng2024hybridflow}, as detailed below.
At each training step, $k=8$ responses are sampled per prompt to estimate advantages and finetune models, using a temperature of $1.0$ and $\texttt{top\_p}=1.0$. Performance is measured with pass@1, computed across multiple independent generations for each prompt, with the number of samples varying by benchmark. 
For evaluation, we follow \citet{luo2025deepscaler,qu2025can} and generate responses with a temperature of $0.6$ and $\texttt{top\_p}=0.95$.
We remove the KL-divergence penalty from the policy optimization loss, consistent with \citet{yu2025dapo}. In addition, all methods employ the clip-higher technique from DAPO~\citep{yu2025dapo}, which uses asymmetric clipping bounds with $\epsilon_{\text{low}}=0.2$ and $\epsilon_{\text{high}}=0.28$ to enhance exploration.

The training batch size $B$ is set to 256 for both DeepScaleR and Countdown, with mini-batch sizes of 128 and 64, respectively. For Geometry, $B$ is increased to 512 with a mini-batch size of 256. The maximum output length is set to 8192 tokens for DeepScaleR and 1024 tokens for Countdown and Geometry.
Entropy regularization is applied to Countdown with a coefficient of 0.001, while it is disabled for DeepScaleR and Geometry. Models is optimized using Adam~\citep{kingma2014adam} with momentum parameters $\beta=(0.9, 0.999)$ and a weight decay of $0.01$. The learning rate is set to $1\mathrm{e}{-6}$ for Countdown and Geometry. For DeepScaleR, we use $2\mathrm{e}{-6}$ for the 7B model and $4\mathrm{e}{-6}$ for the 1.5B model. DeepScaleR adopts a binary reward function that assigns a value of 1 to correct answers and 0 otherwise, following the default configuration in verl~\citep{sheng2024hybridflow}. In contrast, Countdown and Geometry3k incorporate an additional format bonus of 0.1 when the response is incorrectly solved but adheres to the required formatting, consistent with the setup in \cite{tinyzero}.
All experiments are executed on 8 NVIDIA A100 GPUs with 80GB memory.

\begin{figure}[t]
    \centering
    \includegraphics[width=\linewidth]{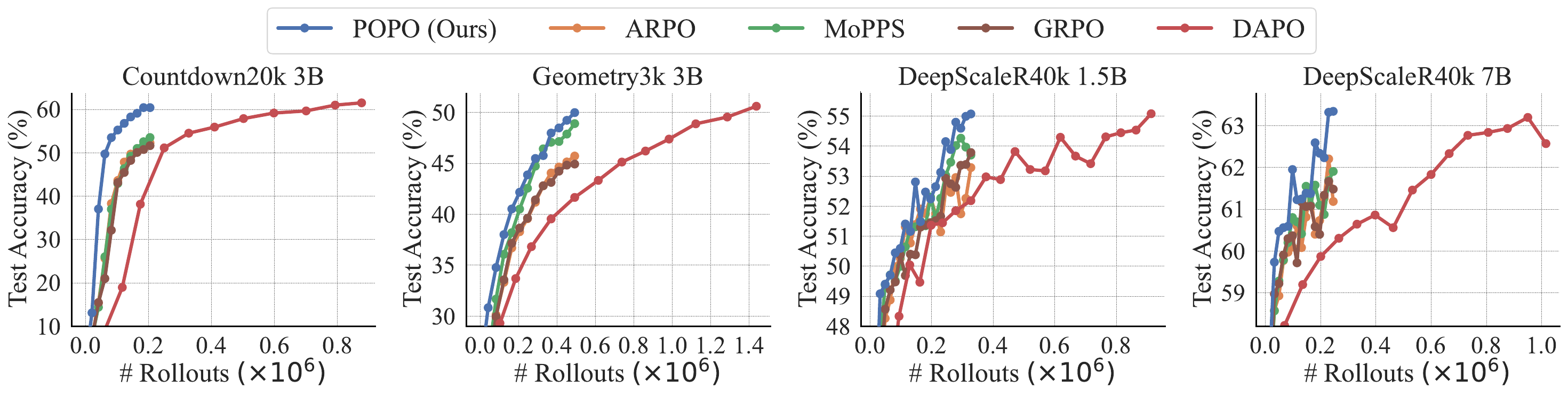}
    \vspace{-15pt}
    \caption{
    Training curves of different methods across tasks as a function of LLM rollout count.
    }
    \label{fig:performance_rollout}
\end{figure}

\section{Extended Experimental Results}
\label{appsec:results}

\subsection{Rollout Efficiency}
\cref{fig:performance} shows that \algo and DAPO substantially accelerate RL fine-tuning compared to the other baselines when measured by training steps.
However, training-step-based comparisons fail to account for the cost of LLM inference, which often dominates the overall training expense.
Because DAPO relies on oversampling, it generates a larger number of rollouts per step, leading to significantly higher inference overhead.
\cref{fig:performance_rollout} instead plots performance as a function of the cumulative number of rollouts.
The results demonstrate that \algo achieves strong performance using markedly fewer rollouts than DAPO, typically requiring 30\% of DAPO's rollout budget to reach comparable results.

\subsection{Integration with Sampling Methods}

\begin{figure}[t]
    \centering
    \includegraphics[width=0.9\linewidth]{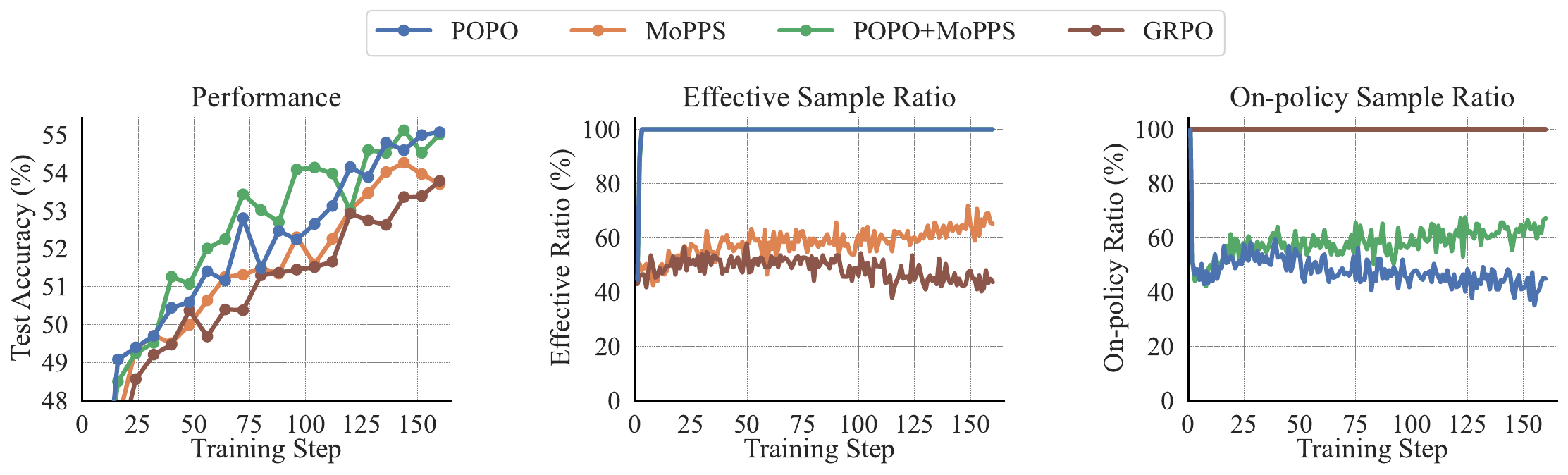}
    \caption{
    Test performance, effective sample ratio, and on-policy sample ratio on the DeepScaleR 1.5B task when integrating \algo with the predictive sampling method MoPPS.
    }
    \label{appfig:sampling}
\end{figure}

\algo can be naturally combined with sampling methods that adaptively select informative prompts. By incorporating predictive sampling (e.g., MOPPS), \algo increases the likelihood of sampling effective groups online, thereby reducing the need to fill training batches with off-policy effective groups and potentially yielding further performance gains. To investigate this effect, we integrate MOPPS with \algo, denoted as POPO+MOPPS, and compare it against standalone \algo, MOPPS, and GRPO on the DeepScaleR 1.5B task. Fig.~\ref{appfig:sampling} reports the full results, including test performance, effective sample ratio, and on-policy sample ratio throughout training. 

As shown in Fig.~\ref{appfig:sampling}, \algo+MOPPS maintains a 100\% effective sample ratio, the same as \algo, while achieving a higher on-policy sample ratio. This shift toward more on-policy data leads to additional performance gains, as evidenced by faster improvement and higher test accuracy. These results indicate that combining predictive sampling with \algo enables a more favorable balance between on-policy and off-policy data, yielding complementary benefits for policy optimization.

\subsection{Response Length}

\begin{figure}[t]
    \centering
    \includegraphics[width=\linewidth]{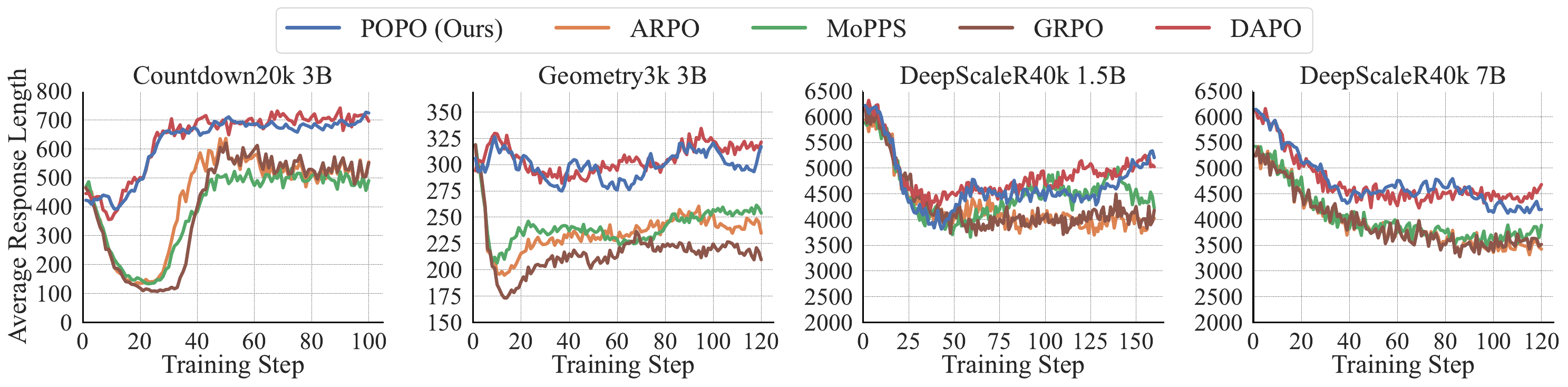}
    \caption{Comparison of the average response length across different RLVR methods.
    }
    \label{fig:responselength}
\end{figure}

Response length has been identified as a strong correlate of reasoning capability~\citep{yu2025dapo}.
\cref{fig:responselength} illustrates how different methods affect this metric over the course of training.
Across tasks, the average response length of \algo follows a trajectory similar to that of DAPO, whereas ARPO and MoPPS exhibit trends more closely aligned with GRPO.
This suggests that \algo leverages off-policy data to induce training dynamics comparable to those of DAPO, which relies solely on online data.
Both \algo and DAPO consistently produce longer responses than the remaining baselines.
Such extended outputs encourage deeper exploration and support more elaborate reasoning processes, partly explaining the observed performance differences~\citep{yu2025dapo}.

\section{Broader Societal Impacts}
\label{appsec:broader}

This work presents POPO, a sample-efficient framework for RL finetuning of LLMs. By replacing ineffective training data with high-quality off-policy samples, POPO reduces the number of model rollouts required during training, thereby lowering computational overhead and associated energy consumption. 
Improved efficiency at this stage has the potential to broaden access to the development of advanced reasoning models, particularly for researchers and organizations with limited computational resources.
At the same time, as with any method that enhances the capabilities of LLMs, there remains the possibility of misuse. However, POPO does not introduce fundamentally new risks beyond those already inherent in existing RLVR paradigms; rather, it improves the efficiency of established training processes without altering their core objectives or deployment characteristics.
The societal impact of such technology ultimately depends on the specific downstream applications and the safeguards employed during deployment.
We therefore emphasize the importance of continued research into responsible AI practices to ensure that gains in training efficiency are directed toward positive and broadly beneficial outcomes.

\section{Data Examples}
\label{appsec:dataexample}
We present representative examples of the data used across all tasks in our experiments.
For the DeepScaleR dataset and mathematics benchmarks, prompts are constructed by appending a chain-of-thought instruction~\citep{wei2022chain} along with a formatting constraint:
``Let's think step by step and output the final answer within \textbackslash boxed\{\}''.
For general reasoning benchmarks, we follow the evaluation protocol of \citet{yan2025learning} and adopt the PRIME prompt template.
For the Countdown task, we use the prompt format introduced by \citet{tinyzero}. For the Geometry task, we adopt the prompt templates from verl~\citep{sheng2024hybridflow}.

\begin{tcolorbox}[
  colback=white,
  colframe=black,
  boxrule=0.5pt,
  arc=0pt,
  colbacktitle=white,
  coltitle=black,
  fonttitle=\bfseries,
  title=DeepscaleR Data Example
]
\textbf{Prompt:}

Consider all 1000-element subsets of the set $\{1, 2, 3, \dots , 2015\}$. From each such subset choose the least element. The arithmetic mean of all of these least elements is $\frac{p}{q}$, where $p$ and $q$ are relatively prime positive integers. Find $p + q$.
Let's think step by step and output the final answer within \texttt{\textbackslash boxed\{\}}.

\textbf{Answer:}

2016
\end{tcolorbox}

\begin{tcolorbox}[
  colback=white,
  colframe=black,
  boxrule=0.5pt,
  arc=0pt,
  colbacktitle=white,
  coltitle=black,
  fonttitle=\bfseries,
  title=Countdown Data Example
]
\textbf{Prompt:}

A conversation between User and Assistant. The user asks a question, and the Assistant solves it. The assistant first thinks about the reasoning process in the mind and then provides the user with the answer.

User: Using the numbers [95, 11, 56], create an equation that equals 28. You can use basic arithmetic operations (+, -, *, /) and each number can only be used once. Show your work in $<$think$>\ </$think$>$ tags. And return the final answer in $<$answer$>\ </$answer$>$ tags, for example $<$answer$> (1 + 2) / 3 <$/answer$>$.

Assistant: Let me solve this step by step.

$<$think$>$

\end{tcolorbox}
\begin{tcolorbox}[
  colback=white,
  colframe=black,
  boxrule=0.5pt,
  arc=0pt,
  colbacktitle=white,
  coltitle=black,
  fonttitle=\bfseries,
  title=Geometry3k Data Example
]
\textbf{Prompt:}

\includegraphics[width=0.25\linewidth]{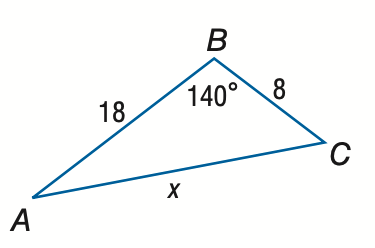}
Find x. Round to the nearest tenth.
You FIRST think about the reasoning process as an internal monologue and then provide the final answer. The reasoning process MUST BE enclosed within $<$think$>\ </$think$>$ tags. The final answer MUST BE put in \texttt{\textbackslash boxed\{\}}.

\textbf{Answer:}

24.7
\end{tcolorbox}

\begin{tcolorbox}[
  colback=white,
  colframe=black,
  boxrule=0.5pt,
  arc=0pt,
  colbacktitle=white,
  coltitle=black,
  fonttitle=\bfseries,
  title=MMLU-Pro Data Example
]
\textbf{Prompt:}

The following are multiple choice questions (with answers) about {\$}. Think step by step and then finish your answer with "\texttt{\textbackslash boxed\{X\}}" where X is the correct letter choice.

Question:

Managers are entrusted to run the company in the best interest of \_\_\_\_\_\_\_\_. Specifically, they have a duty to act for the benefit of the company, as well as a duty of \_\_\_\_\_\_\_\_ and of \_\_\_\_\_\_\_.

Options:

A. Shareholders, Diligence, Self-interest

B. Shareholders, Self-interest, Care and Skill

C. Stakeholders, Care and skill, Self-interest

D. Stakeholders, Diligence, Care and Skill

E. Customers, Care and Skill, Diligence

F. Shareholders, Care and Skill, Diligence

G. Shareholders, Self-interest, Diligence

H. Employees, Care and Skill, Diligence

I. Stakeholders, Self-interest, Diligence

J. Stakeholder, Care and Skill, Diligence

\textbf{Answer:}

F
\end{tcolorbox}

\begin{tcolorbox}[
  colback=white,
  colframe=black,
  boxrule=0.5pt,
  arc=0pt,
  colbacktitle=white,
  coltitle=black,
  fonttitle=\bfseries,
  title=GPQA-Diamond Data Example
]
\textbf{Prompt:}

The following are multiple choice questions (with answers) about {\$}. Think step by step and then finish your answer with "\texttt{\textbackslash boxed\{X\}}" where X is the correct letter choice.

Question:

Astronomers are studying a system of three exoplanets (Planet1, Planet2, and Planet3) with circular orbits discovered through the TTV method. They have found that the ratio of the equilibrium temperatures between Planet1 and Planet2 is approximately 1.4, and between Planet2 and Planet3, it is about 2.3. They have also found that the ratio of the masses between Planet1 and Planet2 is approximately 1.15, and between Planet2 and Planet3, it is about 1.35. By what factor is the orbital period of Planet3 larger than that of Planet1, if the albedo for all three planets is equal to 0.3 (similar to that of Earth)?

Options:

A. $\sim$ 3.2

B. $\sim$ 4.4

C. $\sim$ 10.4

D. $\sim$ 33.4

\textbf{Answer:}

D
\end{tcolorbox}

\begin{tcolorbox}[
  colback=white,
  colframe=black,
  boxrule=0.5pt,
  arc=0pt,
  colbacktitle=white,
  coltitle=black,
  fonttitle=\bfseries,
  title=ARC-C Data Example
]
\textbf{Prompt:}

The following are multiple choice questions (with answers) about {\$}. Think step by step and then finish your answer with "\texttt{\textbackslash boxed\{X\}}" where X is the correct letter choice.

Question:

George wants to warm his hands quickly by rubbing them. Which skin surface will produce the most heat?

Options:

A. dry palms

B. wet palms

C. palms covered with oil

D. palms covered with lotion

\textbf{Answer:}

A
\end{tcolorbox}

\end{document}